\documentclass{article}

\usepackage{microtype}
\usepackage{graphicx}
\usepackage{subcaption}
\usepackage{booktabs} 
\usepackage[normalem]{ulem}
\usepackage{pgfplots}
\usepackage{enumitem}
\usepackage{multirow}
\usepackage[hidelinks]{hyperref}

\usetikzlibrary{calc}

\usepackage{hyperref}

\usepackage[accepted]{icml2023}

\usepackage{amsmath}
\usepackage{amssymb}
\usepackage{mathtools}
\usepackage{amsthm}

\DeclareMathOperator{\partialt}{\partial_t}
\DeclareMathOperator{\diag}{diag}
\def\xit{\ensuremath{\mathbf X_i}}
\def\xt{\ensuremath{\mathbf X}}
\def\dxit{\ensuremath{\nabla_{\mathbf X_i}}}
\def\dxt{\ensuremath{\nabla_{\mathbf X}}}

\usepackage[capitalize,noabbrev]{cleveref}

\theoremstyle{plain}
\newtheorem{theorem}{Theorem}[section]
\newtheorem{proposition}[theorem]{Proposition}
\newtheorem{lemma}[theorem]{Lemma}

\theoremstyle{definition}

\theoremstyle{remark}

\usepackage[textsize=tiny]{todonotes}

\icmltitlerunning{Entropy Aware Message Passing in Graph Neural Networks}

\begin{document}

\twocolumn[
\icmltitle{Entropy Aware Message Passing in Graph Neural Networks}



\icmlsetsymbol{equal}{*}

\begin{icmlauthorlist}
\icmlauthor{Philipp Nazari}{equal,eth}
\icmlauthor{Oliver Lemke}{equal,eth}
\icmlauthor{Davide Guidobene}{equal,eth}
\icmlauthor{Artiom Gesp}{equal,eth}
\end{icmlauthorlist}

\icmlaffiliation{eth}{Deep Learning, ETH Zurich, Switzerland}

\icmlcorrespondingauthor{Philipp Nazari}{pnazari@student.ethz.ch}
\icmlcorrespondingauthor{Oliver Lemke}{olemke@student.ethz.ch}
\icmlcorrespondingauthor{David Guidobene}{dguidobene@student.ethz.ch}
\icmlcorrespondingauthor{Artiom Gesp}{agesp@student.ethz.ch}

\icmlkeywords{Machine Learning, ICML}

\vskip 0.3in
]



\printAffiliationsAndNotice{\icmlEqualContribution} 

\begin{abstract} Deep Graph Neural Networks struggle with oversmoothing.
This paper introduces a novel, physics-inspired GNN model designed to mitigate this issue.
Our approach integrates with existing GNN architectures, introducing an entropy-aware message passing term.
This term performs gradient ascent on the entropy during node aggregation, thereby preserving a certain degree of entropy in the embeddings.
We conduct a comparative analysis of our model against state-of-the-art GNNs across various common datasets.
\end{abstract}

\section{Introduction}
Graph Neural Networks have proven to be a powerful instrument for modeling relationships and dependencies in graph-structured data, making them well-suited for tasks such as recommendation systems~\cite{fan2019graph} and molecular chemistry~\cite{Wu2023}.
Despite their strengths, a significant limitation of deep GNNs is the tendency for node embeddings to collapse, a phenomenon known as "graph oversmoothing"~\cite{rusch2023survey}.
Oversmoothing leads to drastically worse performance, severely limiting the effectiveness especially of deep GNNs.
This paper addresses the oversmoothing problem by proposing an adaptation of existing GNN frameworks.
We introduce an entropy-aware message passing mechanism, which encourages the preservation of entropy in graph embeddings.



In a comparative study,~\citet{rusch2023survey} find that gradient-gating~\cite{rusch2022gradient} can effectively mitigate oversmoothing while still allowing for expressivity of the network.
However, it is noted that gradient gating may inadvertently constrain the GNN's expressiveness by rendering embeddings static prematurely, thus avoiding oversmoothing.
We believe that such a strong inductive bias may not be essential.
Intuitively, representations can be refined, allowing for the convergence of related node embeddings while ensuring that unrelated nodes remain distinct.

While some work, besides gradient-gating, has focused on tackling the problem of oversmoothing explicitly via architectural design, other approaches like PairNorm~\cite{zhao2019pairnorm} approach the topic through a lens of regularization.
PairNorm aims to maintain a constant total sum of pairwise distances in the embedding space during transformations.
However, to maintain such a strict constraint, PairNorm sacrifices model's expressivity as shown in \cite{rusch2023survey}.
In an attempt to avoid this issue, we propose a softened version of the same idea: encouraging the entropy of the embedding not to collapse to zero. 
This bears similarity to RankMe, proposed by \cite{novikova2018rankme}, who use Shannon entropy to measure representational collapse in self-supervised learning.

More explicitly, we define an unnormalized Boltzmann distribution over the node energies (Section~\ref{sec:entropy}). We provide a closed form expression of the corresponding Shannon entropy gradient (Theorem~\ref{theorem:ds}), which facilitates a gradient ascent process in each GNN layer (Section~\ref{sec:entropic-message-passing}).

In Section~\ref{sec:results}, we evaluate our approach by comparing to a standard GCN, as well as the two baselines PairNorm~\cite{zhao2019pairnorm} and G2~\cite{rusch2022gradient}.
The evaluation is conducted on Cora~\cite{mccallum2000automating} and CiteSeer~\cite{giles1998citeseer} , as well as the 2D grid dataset introduced by \citet{rusch2023survey}.

We find that, while entropic GCN does alleviate oversmoothing similarly well as PairNorm~\cite{zhao2019pairnorm} and G2~\cite{rusch2022gradient}, it does not maintain competitive accuracy for deeper networks. Similar to the work of~\cite{rusch2023survey}, this shows that solving oversmoothing is necessary, but not sufficient for training deep networks.
%
%
We
\begin{itemize}[leftmargin=15pt, labelindent=0pt, itemsep=1pt, parsep=1pt, topsep=1pt, partopsep=1pt]
    \item provide a framework for existing GNN architectures that allows for entropy aware message passing by doing gradient ascent on the entropy at each layer,
    \item give a closed form expression of the entropy gradient
    \item evaluate our model against a number of baselines.
\end{itemize}

The code is made available at \href{https://github.com/oliver-lemke/gnn_dl}{https://github.com/oliver-lemke/entropy\_aware\_message\_passing}.

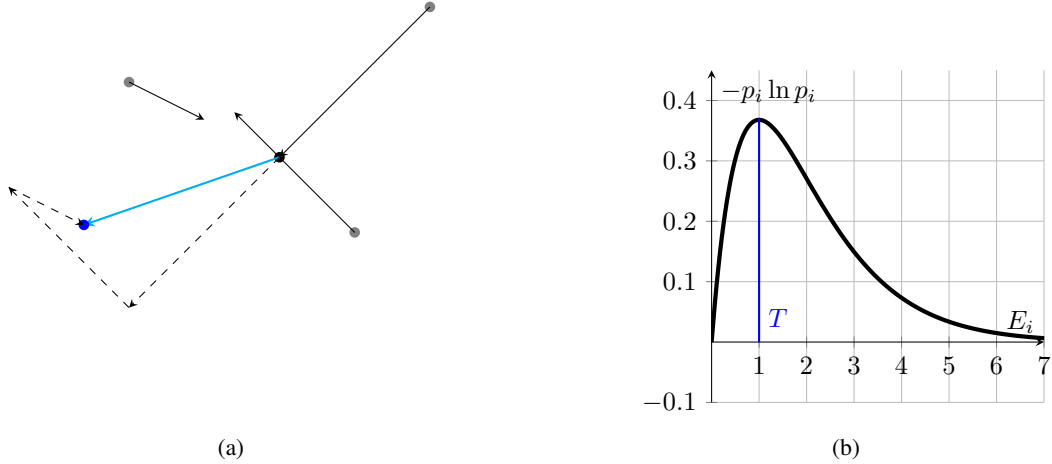
\begin{figure*}[ht]
    \centering
    \begin{subfigure}[b]{0.35\textwidth}
        \centering
        \begin{tikzpicture}
            \coordinate (w) at (0,0);
            \coordinate (x) at (1,-1);
            \coordinate (xx) at (-0.6,0.6);  
            \coordinate (y) at (-2,1);
            \coordinate (yy) at (-1,0.5);  
            \coordinate (z) at (2,2);
        
            \foreach \point in {x,y,z}
                \fill[color=gray] (\point) circle (2pt) node[above right] {};
        
            \fill[color=black] (w) circle (2pt) node[above right] {};
        
            \draw[->, >=stealth] (x) -- (xx);
            \draw[->, >=stealth] (y) -- (yy);
            \draw[->, >=stealth] (z) -- (w); 
        
            \coordinate (result) at ($ -1*(x) + (xx) - (y) + (yy) -(z)$);
            \draw[->, >=stealth, thick, cyan] (w) -- (result) node[midway, above right, black] {};
        
            \fill[color=blue] (result) circle (2pt) node[above right] {};
        
            \draw[->, >=stealth, dashed] (w) -- ($-1*(z)$) node[midway, above right, black] {};
            \draw[->, >=stealth, dashed] ($-1*(z)$) -- ($-1*(x) + (xx) -1*(z)$) node[midway, above right, black] {};
            \draw[->, >=stealth, dashed] ($-1*(x) + (xx) -1*(z)$) -- (result) node[midway, above right, black] {};
        \end{tikzpicture}
        \vspace{1.5cm}
        \caption{}
        \label{fig:entropy-gradient}
    \end{subfigure}
    \hspace{2cm}
    \begin{subfigure}[b]{0.35\textwidth}
        \centering
        \begin{tikzpicture}
            \begin{axis}[
                xlabel=$E_i$,
                ylabel={$-p_i \ln p_i$},
                xmin=0, xmax=7,
                ymin=-0.1, ymax=0.45,
                grid=both,
                axis lines=middle,
                width=\textwidth,
                height=\textwidth,
                xtick={0,1,2,3,4,5,6,7},
                ytick={-0.1,0,0.1,0.2,0.3,0.4},
                xticklabels={$0$, $1$, $2$, $3$, $4$, $5$, $6$, $7$},
                yticklabels={$-0.1$, $0$, $0.1$, $0.2$, $0.3$, $0.4$},
                domain=0:7,
                samples=100,
                smooth,
                ]
                \addplot[black,ultra thick] {-exp(-x)*(-x)};
    
                \draw[blue, thick] (axis cs:1,0.0) -- (axis cs:1,0.37);
                \node[blue] at (axis cs:1.4,0.04) {$T$};
                
            \end{axis}
        \end{tikzpicture}
        \caption{}
        \label{fig:plnp}
    \end{subfigure}
    \caption{
    Panel~(\subref{fig:entropy-gradient}) explains the effect of gradient ascent during aggregation, where a node \textcolor{black}{$\bullet$} is pushed in a direction that is a weighted superposition \textcolor{cyan}{$\rightarrow$} of the vectors pointing from its neighbors \textcolor{gray}{$\bullet$}, in a way that leads to a maximal increase in entropy.
    Panel~(\subref{fig:plnp}) shows the contribution to the entropy of a single node $i$, as a function of energy $E_i$ at temperature $T$. The contribution is maximized iff $E_i = T$.}
    \label{fig:one}
\end{figure*}

\section{Constructing An Entropy}
\label{sec:entropy}
In this section, we construct an entropy. It will be used in Section~\ref{sec:entropic-message-passing} to define the entropy aware node aggregation mechanism.

Assume a bi-directional, unweighted graph $\mathcal G = (\mathcal V, \mathcal E)$, with $n \coloneqq |\mathcal V|$ nodes.
We define the Dirichlet energy of a $d$-dimensional graph embedding $\mathbf X \in \mathbb R^{n \times d}$ as
\begin{equation*}
    \label{eq:dirichlet-energy}
    E = \frac 1 {|\mathcal V|} \sum_{i \in \mathcal V} E_i,
\end{equation*}
where $E_i$ is the Dirichlet energy of node $i$,
\begin{align}
    \label{eq:def-ei}
    E_i & = \frac 1 {2 \sqrt {|\mathcal N_i| d}} \sum_{j \in \mathcal N_i} \| \mathbf X_j - \mathbf X_i \|_2^2.
\end{align}
In the following, we will write the normalization constant of $E_i$ as $C_i \coloneqq \frac 1 {\sqrt {|\mathcal N_i| d}}$. This normalization helps reduce variance for high dimensional embeddings~\cite{vaswani2017attention}.

Inspired by physics, we use the \textit{Boltzmann distribution} to assign an ``unnormalized probability`` $p_i$ to node $i$ having energy $E_i$ at fixed temperature $T$:
\begin{equation}
    \label{eq:boltzmann-dist}
    p_i \coloneqq e^{-E_i/T}.
\end{equation}
The temperature $T$ is a free parameter in our setting that will allow us to regulate the entropy regularization later on. Stacking these unnormalized probabilities, we obtain a vector $\mathbf P \in \mathbb R^n$.

The entropy $S$ we consider is now given by the \textit{Shannon entropy} of this Boltzmann distribution:
\begin{equation}
    \label{eq:entropy}
    S\left(\mathbf \xt \right) \coloneqq - \mathbb E_{p_i} \left[ \ln p_i \right] = -\sum_{i \in \mathcal V} p_i \ln\left(p_i\right) \in \mathbb R.
\end{equation}
This term is maximized if and only if $E_i = T$ for all $i \in \mathcal V$. Hence, the temperature $T$ provides a way of controlling how much smoothing we allow for (see Figure~\ref{fig:plnp}).

In order to simplify terms in later expressions, we introduce the auxiliary variable
\begin{equation}
    \label{eq:pbar}
    \overline{\mathbf P} \coloneqq \left( \mathbf P + 1 \right) \odot \ln \mathbf P \in \mathbb R^n.
\end{equation}

\section{Entropy Aware Message Passing}
\label{sec:entropic-message-passing}

Assume any underlying GNN architecture, e.g. message passing or attention based, which updates the embedding of node $i$ from layer $k$ to layer $k + 1$ like
\begin{equation*}
    \mathbf X_i^{(k+1)} = \mathbf X_i^{(k)} + \phi\left(\mathbf X_i^{(k)}, \mathcal N_{i}\right),
\end{equation*}
where $\phi$ is a function that is permutation invariant with respect to the neighbors $\mathcal N_{i}$ of node $i$. We propose extending such an aggregation step by performing gradient ascent on the entropy $S$:
\begin{align}
    \label{eq:de-entropy}
    \mathbf X_i^{(k+1)} &= \mathbf X_i^{(k)} + \phi\left(\mathbf X_i^{(k)}, \mathcal N_{i}\right) + \lambda  \cdot T\nabla_{\mathbf X_i} S\left(\mathbf X^{(k)}_c\right),
\end{align}
where $\lambda > 0$ is a hyperparameter.\footnote{This framework also applies in the setting of Graph Neural Diffusion, see Appendix~\ref{appendix:application-to-graph-neural-diffusion}.} The ``c`` in $S\left(\mathbf X^{(k)}_c\right)$ indicates that we evaluate the entropy on a deep copy of the embedding and thus do not back-propagate through the entropy gradient.
This strategy is adopted to reduce the computational complexity of the method, as direct backpropagation through this process would lead to excessively expensive gradient calculations.
The following theorem gives a closed-form expression for the required gradient of the entropy:

\begin{theorem}
\label{theorem:ds}
    The gradient of the Entropy $S$ with respect to $\xit$ is
    \begin{equation}
        \label{eq:grad_entropy}
        \dxit S\left(\xt\right) = C_j \frac 1 T \sum_{j \in \mathcal N_i} \left( \overline{\mathbf P}_j + \overline{\mathbf P}_i \right) \left(\mathbf X_i - \mathbf X_j \right)
    \end{equation}
\end{theorem}
\begin{proof}
    See Appendix~\ref{section:proof-theorem}.
\end{proof}

Intuitively, $\mathbf X_i - \mathbf X_j$ represents a vector pointing from the embedding of neighbor $j$ towards the embedding of node $i$. Hence, the contribution to the sum pushes node $i$ away from its neighbor $j$ if and only if $\overline{\mathbf P}_j + \overline{\mathbf P}_i > 0$.
Otherwise it pulls $\mathbf X_i$ towards $\mathbf X_j$.

It is easy to see that for any $k = 1,...,n$, it holds $\overline{\mathbf P}_k > 0$ if and only if $E_k < T$.
Whether $\mathbf X_i$ is pushed away from, or pulled towards, its neighbor $\mathbf X_j$ thus depends on which updates moves $E_i$ and $E_j$ towards $T$ most effectively.
For a visual representation of this dynamic, refer to Figure~\ref{fig:entropy-gradient}.

This argument shows that the updated message passing neither allows for converging, nor diverging embeddings. It trades of oversmoothing and divergence of embeddings.

\begin{lemma}
    \label{lemma:complexity}
    The complexity of computing $\dxt S(\xt)$ is $\mathcal O(m + n)$, where $n$ represents the number of nodes and $m$ the number of edges. For sparse graphs $\mathcal G$, which is the common scenario, the complexity simplifies to $\mathcal O(n)$.
\end{lemma}
\begin{proof}
    See Appendix~\ref{sec:proof-lemma-complexity}.
\end{proof}


\section{Related Work}
The problem of oversmoothing has been tackled both from the perspective of regularization~\cite{zhao2019pairnorm, godwin2021simple, zhou2021dirichlet} and architectural design~\cite{rusch2022gradient, li2019deepgcns, chen2020simple}.

On the premise of architectural design, previous work has successfully adapted ResNets \cite{he2016deep} to graphs \cite{li2019deepgcns, chen2020simple}. In another work,~\citet{rusch2022gradient} try to prevent deep GNNs from oversmoothing by gating the output of each GNN layer, making sure that embeddings become stationary before they could become too similar.
In the closely related field of graph neural diffusion~\cite{chamberlain2021grand}, which our method can be adapted for,~\citet{wang2022acmp} adapt the message passing procedure to allow for a repulsive term by introducing a repulsive potential.

Taking the perspective of regularization,~\cite{zhou2021dirichlet} directly lower the effect of oversmoothing by optimizing the GNN withing a constrained range of the Dirichlet energy.

Entropy is a reoccurring quantity across different domains of machine learning. In self-supervised learning~\cite{novikova2018rankme} it is used for measuring representational collapse, in reinforcement learning~\cite{haarnoja2018soft} to support exploration. In both mentioned applications, entropy is a proxy measure for complexity, where vanishing entropy corresponds to trivial complexity.
Our entropy aware approach bears resemblance to PairNorm~\cite{zhao2019pairnorm}, who propose keeping the total sum of pairwise distances in the embedding constant. We soften this idea by not fixing the total distance, but taking a step into the direction of higher entropy.

\begin{figure}
    \centering
    \includegraphics[width=\linewidth]{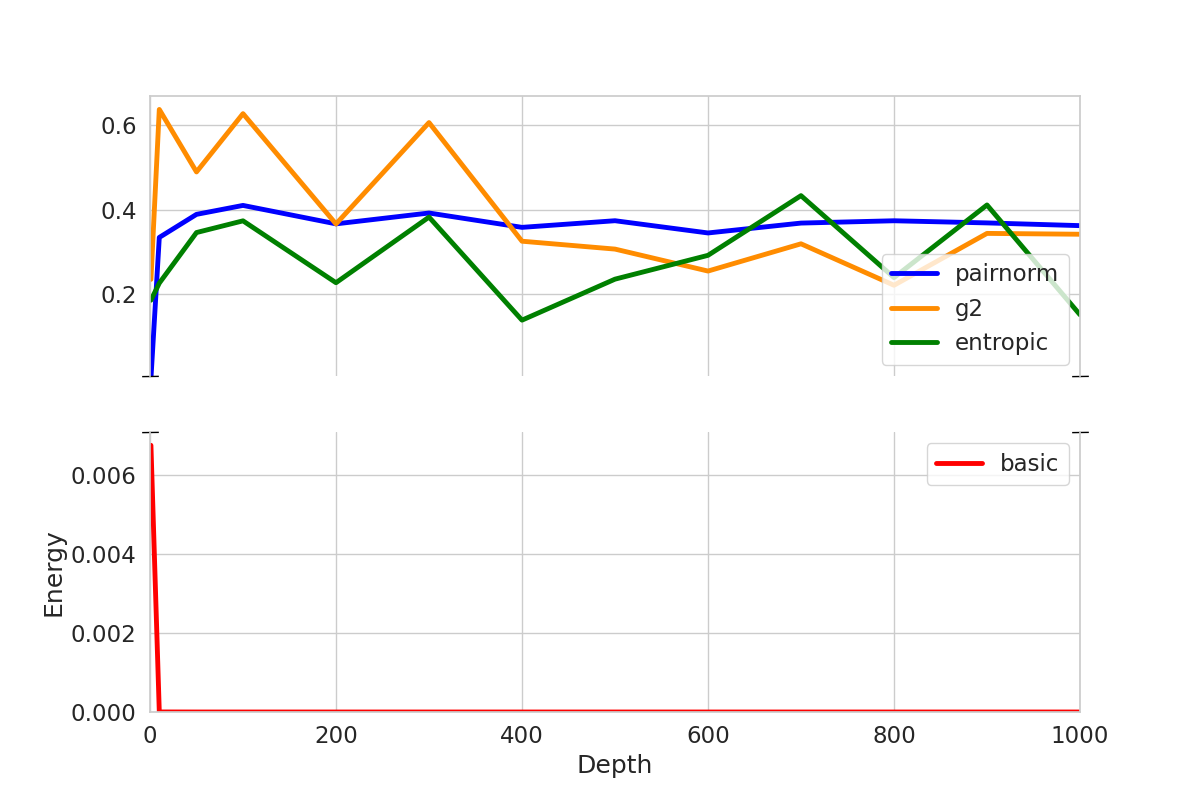}
    \caption{Energy as a function of depth evaluated on the nearest neighbor graph of a randomly initialized $10 \times 10$ grid. Entropic GCN ensures constant energy for depth up to $1000$, while basic GCN oversmooths quickly.}
    \label{fig:message-passing}
\end{figure}

\section{Results}
\label{sec:results}
\paragraph{Hyperparameter Selection in Entropic GCN.}
The Entropic GCN relies on two principal hyperparameters: the entropy gradient weight $\lambda$ (refer to Eq. \ref{eq:de-entropy}) as well as the temperature $T$ (see Eq. \ref{eq:grad_entropy}).
Following extensive hyperparameter tuning on Cora, we determined optimal values $\lambda = 1, T = 10$ for all experiments, except for training on CiteSeer, where $\lambda = 10, T = 1$ was chosen.
We observe that the selection of these hyperparameters exhibits significant sensitivity to the specific task being addressed.
To ensure a balanced comparison, we set the width of all hidden layers to 256 across all models.

As proposed by~\citet{rusch2023survey}, we evaluate our models influence on oversmoothing on an artificial toy dataset given by the nearest neighbor graph of a two-dimensional $10 \times 10$ grid with randomly initialized embedding. We evaluate the Dirichlet energy of untrained models at different depths and compare it to existing baselines. The result is shown in Figure~\ref{fig:message-passing}.
Indeed, our results confirm that, similar to PairNorm~\cite{zhao2019pairnorm} and G2~\cite{rusch2022gradient}, entropy aware message passing alleviates oversmoothing.
However, our experiments also show that solving oversmoothing is not sufficient for well performing, deep models. Training an entropic GCN on Cora~\cite{mccallum2000automating} and CiteSeer~\cite{giles1998citeseer}, Table~\ref{tab:cora} shows that, while performance does not drop as quickly as for basic GCN, the accuracy does degrade with deeper networks.
For a depth of 64 on Cora, for example, entropic GCN has approximately double the accuracy as basic GCN at $.35$ vs $.18$. 
However, for shallow networks, our model outperforms both PairNorm and G2.
For the latter we observe approximately constant accuracy of $.55$. We could not reproduce the accuracy stated in \citet{rusch2022gradient}. 


\begin{table}[]
    \centering
    \begin{tabular}{llccccc}
        \toprule
             & Depth & 4 & 8 & 16 & 32 & 64 \\
            Dataset & Model &  &  &  &  &  \\
        \midrule
            \multirow{4}{*}{Cora} & Basic & .82 & .79 & .64 & .42 & .18 \\
             & Entropic & .81 & .79 & .71 & .41 & .35 \\
             & G2 & .54 & .55 & .55 & .56 & .56 \\
             & PairNorm & .77 & .76 & .73 & .72 & .72 \\
        \cline{1-7}
            \multirow{4}{*}{CiteSeer} & Basic & .69 & .64 & .56 & .37 & .29 \\
             & Entropic & .68 & .65 & .61 & .46 & .42 \\
             & G2 & .50 & .50 & .54 & .54 & .53 \\
             & PairNorm & .63 & .56 & .50 & .51 & .50 \\
        \cline{1-7}
        \bottomrule
    \end{tabular}
    \caption{Accuracy of node classifiers of different depths trained on Cora. While entropic GCN does slow down the drop in accuracy, it is not competitive with SoTA models like G2 and PairNorm.}
    \label{tab:cora}
\end{table}

To further investigate where oversmoothing is happening, we plot the energy at each layer of trained models in Figure~\ref{fig:u-shape} (``U-curves``). Basic GCN has high energy in early and late layers ($E \approx 6$), and low energy in the remaining middle layers ($E \approx 0$).
Furthermore, we observe that the minimum energy for basic GCN drops to $0.02$, while entropic GCN consistently maintains an energy above $0.09$.
The U-curve of PairNorm is flatter than that of basic and entropic GCN, and for G2 it even increases monotonically. Both models alleviate oversmoothing (Figure~\ref{fig:message-passing}) and maintain good performance for up to $64$ layers (Table~\ref{tab:cora}). 
Consequently, it seems like oversmoothing is predominantly happening in intermediate layers, which explains why entropic GCN is not competitive with state of the art models. However, compared to basic GCN, our model recovers the energy more rapidly in the last layers. 

\begin{figure}
    \centering
    \includegraphics[width=\linewidth]{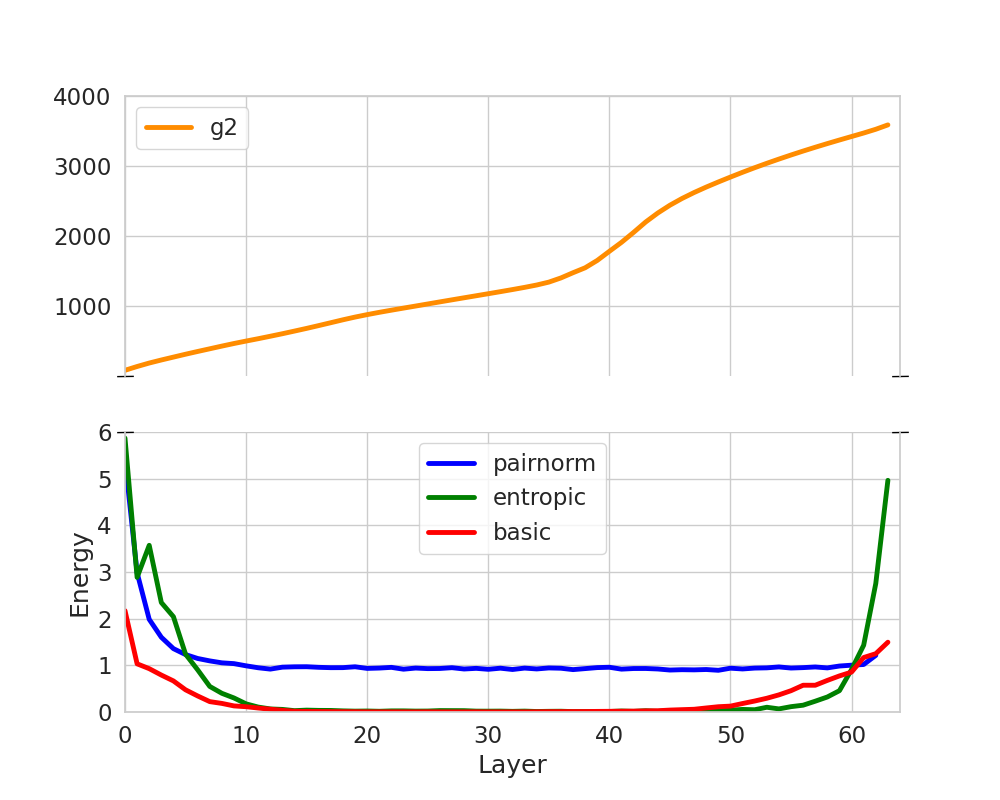}
    \caption{Energy at each layer for models trained on Cora. The U-shape of the basic GCN's curve could suggest that graph neural networks oversmooth in intermediate layers.}
    \label{fig:u-shape}
\end{figure}

\section{Discussion}
\textbf{Summary} In this work, we propose tackling the problem of oversmoothing in graph neural networks by performing a gradient ascent step on the entropy at every layer.
One major strength of this approach is its flexibility, as it independent of both architecture and loss function of the original model.


Our experiments on message passing in untrained models show that entropy aware node aggregation does alleviate oversmoothing comparably well to existing models.
Concerning classification performance, we see that accuracy of entropic GCN drops for deep networks, though it improves significantly upon basic GCN.


In shallow networks, entropic GCN maintains similar accuracy as Basic GCN.
This shows that we achieve our goal of relaxing the hard constraints of PairNorm and G2 without sacrificing expressivity.
However, our approach does not completely solve oversmoothing as indicated by poor accuracy of deeper networks, possibly stemming from low energy levels in intermediate layers.
A possible explanation is that the weight of entropy ascent is not large enough at these stages.

\textbf{Future Work}
One possible line of work would be the development of a weight scheduler that increases $\lambda$ whenever the entropy of the model decreases.
This might enhance performance by assigning greater importance to entropy ascent in intermediate layers, potentially flattening the U-curve.

Further investigation into the causes behind the U-shaped energy curves observed in GNNs also presents a promising research direction.
Deeper understanding this phenomenon could help design novel graph neural network architectures.

Finally, the gradient ascent step on the entropy is currently detached from the computational graph. A more efficient implementation would enable the entropy step to actively contribute gradients to the node embeddings.

\clearpage\newpage

\bibliography{bibliography}

\begin{thebibliography}{17}
\providecommand{\natexlab}[1]{#1}
\providecommand{\url}[1]{\texttt{#1}}
\expandafter\ifx\csname urlstyle\endcsname\relax
  \providecommand{\doi}[1]{doi: #1}\else
  \providecommand{\doi}{doi: \begingroup \urlstyle{rm}\Url}\fi

\bibitem[Chamberlain et~al.(2021)Chamberlain, Rowbottom, Gorinova, Bronstein, Webb, and Rossi]{chamberlain2021grand}
Chamberlain, B., Rowbottom, J., Gorinova, M.~I., Bronstein, M., Webb, S., and Rossi, E.
\newblock Grand: Graph neural diffusion.
\newblock In \emph{International Conference on Machine Learning}, pp.\  1407--1418. PMLR, 2021.

\bibitem[Chen et~al.(2020)Chen, Wei, Huang, Ding, and Li]{chen2020simple}
Chen, M., Wei, Z., Huang, Z., Ding, B., and Li, Y.
\newblock Simple and deep graph convolutional networks.
\newblock In \emph{International conference on machine learning}, pp.\  1725--1735. PMLR, 2020.

\bibitem[Fan et~al.(2019)Fan, Ma, Li, He, Zhao, Tang, and Yin]{fan2019graph}
Fan, W., Ma, Y., Li, Q., He, Y., Zhao, E., Tang, J., and Yin, D.
\newblock Graph neural networks for social recommendation.
\newblock In \emph{The world wide web conference}, pp.\  417--426, 2019.

\bibitem[Giles et~al.(1998)Giles, Bollacker, and Lawrence]{giles1998citeseer}
Giles, C.~L., Bollacker, K.~D., and Lawrence, S.
\newblock Citeseer: An automatic citation indexing system.
\newblock In \emph{Proceedings of the third ACM conference on Digital libraries}, pp.\  89--98, 1998.

\bibitem[Godwin et~al.(2021)Godwin, Schaarschmidt, Gaunt, Sanchez-Gonzalez, Rubanova, Veli{\v{c}}kovi{\'c}, Kirkpatrick, and Battaglia]{godwin2021simple}
Godwin, J., Schaarschmidt, M., Gaunt, A., Sanchez-Gonzalez, A., Rubanova, Y., Veli{\v{c}}kovi{\'c}, P., Kirkpatrick, J., and Battaglia, P.
\newblock Simple gnn regularisation for 3d molecular property prediction \& beyond.
\newblock \emph{arXiv preprint arXiv:2106.07971}, 2021.

\bibitem[Haarnoja et~al.(2018)Haarnoja, Zhou, Hartikainen, Tucker, Ha, Tan, Kumar, Zhu, Gupta, Abbeel, et~al.]{haarnoja2018soft}
Haarnoja, T., Zhou, A., Hartikainen, K., Tucker, G., Ha, S., Tan, J., Kumar, V., Zhu, H., Gupta, A., Abbeel, P., et~al.
\newblock Soft actor-critic algorithms and applications.
\newblock \emph{arXiv preprint arXiv:1812.05905}, 2018.

\bibitem[He et~al.(2016)He, Zhang, Ren, and Sun]{he2016deep}
He, K., Zhang, X., Ren, S., and Sun, J.
\newblock Deep residual learning for image recognition.
\newblock In \emph{Proceedings of the IEEE conference on computer vision and pattern recognition}, pp.\  770--778, 2016.

\bibitem[Li et~al.(2019)Li, Muller, Thabet, and Ghanem]{li2019deepgcns}
Li, G., Muller, M., Thabet, A., and Ghanem, B.
\newblock Deepgcns: Can gcns go as deep as cnns?
\newblock In \emph{Proceedings of the IEEE/CVF international conference on computer vision}, pp.\  9267--9276, 2019.

\bibitem[McCallum et~al.(2000)McCallum, Nigam, Rennie, and Seymore]{mccallum2000automating}
McCallum, A.~K., Nigam, K., Rennie, J., and Seymore, K.
\newblock Automating the construction of internet portals with machine learning.
\newblock \emph{Information Retrieval}, 3:\penalty0 127--163, 2000.

\bibitem[Novikova et~al.(2018)Novikova, Du{\v{s}}ek, and Rieser]{novikova2018rankme}
Novikova, J., Du{\v{s}}ek, O., and Rieser, V.
\newblock Rankme: Reliable human ratings for natural language generation.
\newblock \emph{arXiv preprint arXiv:1803.05928}, 2018.

\bibitem[Rusch et~al.(2022)Rusch, Chamberlain, Mahoney, Bronstein, and Mishra]{rusch2022gradient}
Rusch, T.~K., Chamberlain, B.~P., Mahoney, M.~W., Bronstein, M.~M., and Mishra, S.
\newblock Gradient gating for deep multi-rate learning on graphs.
\newblock \emph{arXiv preprint arXiv:2210.00513}, 2022.

\bibitem[Rusch et~al.(2023)Rusch, Bronstein, and Mishra]{rusch2023survey}
Rusch, T.~K., Bronstein, M.~M., and Mishra, S.
\newblock A survey on oversmoothing in graph neural networks.
\newblock \emph{arXiv preprint arXiv:2303.10993}, 2023.

\bibitem[Vaswani et~al.(2017)Vaswani, Shazeer, Parmar, Uszkoreit, Jones, Gomez, Kaiser, and Polosukhin]{vaswani2017attention}
Vaswani, A., Shazeer, N., Parmar, N., Uszkoreit, J., Jones, L., Gomez, A.~N., Kaiser, {\L}., and Polosukhin, I.
\newblock Attention is all you need.
\newblock \emph{Advances in neural information processing systems}, 30, 2017.

\bibitem[Wang et~al.(2022)Wang, Yi, Liu, Wang, and Jin]{wang2022acmp}
Wang, Y., Yi, K., Liu, X., Wang, Y.~G., and Jin, S.
\newblock Acmp: Allen-cahn message passing for graph neural networks with particle phase transition.
\newblock \emph{arXiv preprint arXiv:2206.05437}, 2022.

\bibitem[Wu et~al.(2023)Wu, Wang, Du, Jiang, Kang, Li, Pan, Deng, Cao, Hsieh, and Hou]{Wu2023}
Wu, Z., Wang, J., Du, H., Jiang, D., Kang, Y., Li, D., Pan, P., Deng, Y., Cao, D., Hsieh, C.-Y., and Hou, T.
\newblock Chemistry-intuitive explanation of graph neural networks for molecular property prediction with substructure masking.
\newblock \emph{Nature Communications}, 14\penalty0 (1):\penalty0 2585, May 2023.
\newblock ISSN 2041-1723.
\newblock \doi{10.1038/s41467-023-38192-3}.
\newblock URL \url{https://doi.org/10.1038/s41467-023-38192-3}.

\bibitem[Zhao \& Akoglu(2019)Zhao and Akoglu]{zhao2019pairnorm}
Zhao, L. and Akoglu, L.
\newblock Pairnorm: Tackling oversmoothing in gnns.
\newblock \emph{arXiv preprint arXiv:1909.12223}, 2019.

\bibitem[Zhou et~al.(2021)Zhou, Huang, Zha, Chen, Li, Choi, and Hu]{zhou2021dirichlet}
Zhou, K., Huang, X., Zha, D., Chen, R., Li, L., Choi, S.-H., and Hu, X.
\newblock Dirichlet energy constrained learning for deep graph neural networks.
\newblock \emph{Advances in Neural Information Processing Systems}, 34:\penalty0 21834--21846, 2021.

\end{thebibliography}
\bibliographystyle{icml2023}

\newpage
\appendix
\onecolumn

\section{Proofs}
\subsection{Proof of Theorem~\ref{theorem:ds}}
\label{section:proof-theorem}
Fix $i,j$ in $\{1,...,n\}$. Our goal is to compute $\dxit \left(\mathbf P_j \ln \mathbf P_j \right)$, so the derivative of the $j$'th summand in the entropy (Equation~\eqref{eq:entropy}) with respect to the $i$'th nodes embedding.
The chain rule gives
\begin{equation}
\label{eq:dxitpjlnpj}
    \dxit \left(\mathbf P_j \ln \mathbf P_j \right) = \mathbf P_j \dxit \ln \mathbf P_j + \dxit \mathbf P_j \ln \mathbf P_j.
\end{equation}
We will look at each of the derivatives separately. For ease of notation, write $e_j = \exp\left(-\frac 1 T E_j \right)$, so that $\mathbf P_j = \frac 1 Z e_j$ (Equation~\ref{eq:boltzmann-dist}).

We start by formulating some helping lemmas that calculate some necessary fundamental derivatives:

\begin{lemma}
    \label{lemma:dxitej}
    We have that
    \begin{equation}
        \dxit e_j = - \frac 1 T e_j \dxit E_j.
    \end{equation}
\end{lemma}
\begin{proof}
    Follows directly from the definition.
\end{proof}

\begin{lemma}
\label{lemma:dxitpj}
    It holds true that
    \begin{equation}
        \dxit \mathbf P_j = \mathbf P_j\left( - \frac 1 T \dxit E_j - \frac 1 Z \dxit Z \right).
    \end{equation}
\end{lemma}
\begin{proof}
    Follows from the chain- and product rules, together with Lemma~\ref{lemma:dxitej}:
    \begin{align*}
        \dxit \mathbf P_j &= \dxit \frac 1 Z e_j \\
        &= \frac 1 Z \dxit e_j + (\dxit \frac 1 Z) e_j \\
        &= -\frac 1 Z \frac 1 T e_j \dxit E_j - e_j \frac 1 {Z^2} \dxit Z \\
        &= -\frac 1 T \mathbf P_j \dxit E_j - \frac 1 Z \mathbf P_j \dxit Z.
    \end{align*}
\end{proof}

\begin{lemma}
    \label{lemma:dxitlnpn}
    It holds true that
    \begin{equation}
        \dxit \ln \mathbf P_j = - \frac 1 T \dxit E_j - \frac 1 Z \dxit Z.
    \end{equation}
    \begin{proof}
        Note that $\ln \mathbf P_j = - \ln Z + \ln e_j = - \ln Z - \frac 1 T E_j$. We thus calculate
        \begin{align*}
            \dxit \ln \mathbf P_j &= \dxit \left(-\ln Z - \frac 1 T E_j\right) \\
            &= - \frac 1 Z \dxit Z - \frac 1 T \dxit E_j.
        \end{align*}
    \end{proof}
\end{lemma}

We now come to a bit less trivial lemma.
\begin{lemma}
    \begin{equation}
    \label{eq:dxitej}
    \nabla_{\mathbf{X_i}} E_j =
    \begin{cases}
    C_j \left(\mathbf X_i - \mathbf X_j\right) & i \in \mathcal N_j \\
    - C_j\sum_{k \in \mathcal N_i} \left(\mathbf X_k - \mathbf X_i\right) & i = j \\
    0 & \text{otherwise}.
    \end{cases}
\end{equation}
\end{lemma}
\begin{proof}
Remember the Equation~\eqref{eq:dirichlet-energy} of the energy $E_j$ related to node $j$:
\begin{align*}
    E_j = \frac 1 2 C_j \sum_{k \in \mathcal N_j} \| \mathbf X_k - \mathbf X_j \|_2^2.
\end{align*}
We compute
\begin{align}
    \dxit E_j &= \frac 1 2 C_j\sum_{k \in \mathcal N_j} \dxit \| \mathbf X_k - \mathbf X_j \|_2^2.
\end{align}
In order to calculate the derivative, we need to differentiate three mutually exclusive cases (note that we do not allow self-loops). We could have that that $i \in \mathcal N_j$, that $i = j$ or neither of both. Assume that $i \in \mathcal N_j$. Then
\begin{equation*}
    \frac 1 2 \sum_{k \in \mathcal N_j} \dxit \| \mathbf X_k - \mathbf X_j \|_2^2 = \mathbf X_i - \mathbf X_j,
\end{equation*}
which is the first case of Equation~\ref{eq:dxitej}.
Next, assume that $i = j$. We compute
\begin{equation*}
    \frac 1 2 \sum_{k \in \mathcal N_j} \dxit \| \mathbf X_k - \mathbf X_j \|_2^2 = - \sum_{k \in \mathcal N_i} \left(\mathbf X_k - \mathbf X_i\right),
\end{equation*}
which shows the second case of Equation~\ref{eq:dxitej}. Finally, if neither of those cases are true, the derivative vanishes.
\end{proof}

With the help of these Lemmas, we can formulate a more high-level Proposition that is a big step towards proving Theorem~\ref{theorem:ds}:
\begin{proposition}
\label{proposition:dxitpjlnpj}
    \begin{equation}
        \dxit \mathbf P_j \ln \mathbf P_j = \left(1 + \ln \mathbf P_j\right) \mathbf P_j \left(-\frac 1 T \dxit E_j \right).
    \end{equation}
\end{proposition}
\begin{proof}
    Insert the results of Lemma~\ref{lemma:dxitpj} and Lemma~\ref{lemma:dxitlnpn} into Equation~\eqref{eq:dxitpjlnpj}.
\end{proof}

We are now finally ready to prove Theorem~\ref{theorem:ds}.
\begin{proof}[Proof of Theorem~\ref{theorem:ds}]
We are interested in
\begin{align*}
    \dxit S(\xt) = -\sum_{j \in \mathcal V} \dxit \mathbf P_j \ln \mathbf P_j.
\end{align*}
Proposition~\ref{proposition:dxitpjlnpj} tells gives us an expression for each of the summands.

Note that
\begin{equation}
\label{eq:2}
    \sum_{j \in \mathcal V} \left(1 + \ln \mathbf P_j\right) \mathbf P_j \left(-\frac 1 T \dxit E_j\right) = - \frac 1 T \sum_{j \in \mathcal N_i \cup i}  \left(1 + \ln \mathbf P_j\right) \mathbf P_j \dxit E_j,
\end{equation}
using Lemma~\ref{lemma:dxitej}.

We thus obtain
\begin{align}
    \dxit S\left(\xt\right) &= \frac 1 T \sum_{j \in \mathcal N_i \cup i} \mathbf P_j \dxit E_j \left(1 + \ln \mathbf P_j)\right)
\end{align}

The last step requires us to replace the remaining gradient $\dxit E_j$ using Lemma~\ref{lemma:dxitej}, using one contribution coming from $i$ and the rest coming from $\mathcal N_j$. Using the fact that $i \in \mathcal N_j$ if and only if $j \in \mathcal N_i$, the latter becomes
\begin{align}
\label{eq:3}
    \frac 1 T \sum_{j \in \mathcal N_i} \mathbf P_j \dxit E_j \left(1 +  \ln \mathbf P_j)\right) = C_j \frac 1 T \sum_{j \in \mathcal N_i} \mathbf P_j \left(1 +  \ln \mathbf P_j)\right) \left(\mathbf X_i - \mathbf X_j\right).
\end{align}
In the remaining case for the contribution coming from $i$, we calculate
\begin{align}
\label{eq:4}
    \frac 1 T \mathbf P_i \dxit E_i \left(1 +  \ln \mathbf P_j)\right) = - C_j\frac 1 T \mathbf P_i \left(1 +  \ln \mathbf P_i)\right) \sum_{j \in \mathcal N_i} \left(\mathbf X_j - \mathbf X_i\right)
\end{align}

Finally, summing over Equations~\eqref{eq:3} and~\eqref{eq:4} gives
\begin{align*}
    \dxit S\left(\xt\right) &= C_j \frac 1 T \sum_{j \in \mathcal N_i} \left(\mathbf X_i - \mathbf X_j\right) \left( \mathbf P_j (1 + \ln \mathbf P_j) + \mathbf P_i (1 + \ln \mathbf P_i) \right).
\end{align*}
Remembering Equation~\eqref{eq:pbar}, which says that $\overline{\mathbf P}_j \coloneqq \mathbf P_j (1 + \ln \mathbf P_j)$, we can rewrite this as
\begin{align}
    \dxit S\left(\xt\right) &= C_j \frac 1 T \sum_{j \in \mathcal N_i} \left(\mathbf X_i - \mathbf X_j\right) \left( \overline{\mathbf P}_j + \overline{\mathbf P}_i \right).
\end{align}

\end{proof}

\subsection{Proof of Lemma~\ref{lemma:complexity}}
\label{sec:proof-lemma-complexity}
    We go through all the required computations step-by-step. First, calculating the energies $E_j$, $j = 1,...,n$, is $\mathcal O(m)$ if $m$ is the number of edges in the graph, which is bounded by $n^2$. Typically, for example in social networks, graphs are sparse, for example $m = \mathcal O(2n)$. Computing $\overline {\mathbf P} = \left(\mathbf P - H \right) \odot \ln \mathbf P$ is also $\mathcal O(n)$. And so is computing $S(\xt) = \sum_{i=1}^n \mathbf P_j \ln \mathbf P_j$. Finally, computing the derivative derived in Proposition~\ref{proposition:dxitpjlnpj} is $\mathcal O(m + n)$, typically $\mathcal O(n)$.

\section{Application To Neural Graph Diffusion}
\label{appendix:application-to-graph-neural-diffusion}
\cite{chamberlain2021grand} interpret message passing neural networks as discrete solutions of a diffusion equation. Concretely, they assume a node $i$ to have (time-dependant) embedding $\mathbf X_i(t) \in \mathbb R^d$. Stacking these representations yields $\mathbf X(t) \in \mathbb R^{n \times d}$.

In this setting, some message passing networks can be thought of solving the differential equation
\begin{equation*}
\label{eq:diffusion}
    \frac \partial \partialt \mathbf X_i(t) = \nabla \left(\mathbf G(\mathbf X_i(t)) \nabla \mathbf X_i(t))\right),
\end{equation*}
where $\mathbf G(\mathbf X(t)) = \diag \left( \mathbf A_{i,j}\right)$ with attention weights $\mathbf A_{i,j} = a(\mathbf X_i, \mathbf X_j)$ and divergence as well as gradient are defined as usually for graphs.

In practice, this differential equation is solved discretely, where every discrete time-step would loosely correspond to a layer in traditional GNNs. During training, one uses a traditional loss to tune the hyperparameters of the attention.

An entropy aware graph neural diffusion model would thus be given by the solution of the following differential equation:
\begin{equation}
    \frac \partial \partialt \mathbf X_i(t) = \nabla \left(\mathbf G(\mathbf X_i(t)) \nabla \mathbf X_i(t))\right) + \lambda \cdot \nabla_{\mathbf X_i(t)} S(\mathbf X(t)).
\end{equation}



\end{document}